\documentclass{article}
\usepackage{ijcai19}

\usepackage{times}
\usepackage{soul}
\usepackage{url}
\usepackage[hidelinks]{hyperref}
\usepackage[utf8]{inputenc}
\usepackage[small]{caption}
\usepackage{graphicx}
\usepackage{amsmath}
\usepackage{booktabs}
\usepackage{amsfonts,amsthm,amssymb} 
\usepackage[ruled,vlined,linesnumbered]{algorithm2e}
\usepackage{multirow}
\usepackage{subfigure}
\usepackage{footnote}
\usepackage{enumerate}
\usepackage{colortbl}
\usepackage{xcolor}
\definecolor{Cyan}{RGB}{0,139,139}
\definecolor{Orange}{cmyk}{.15,.45,1,0}
\definecolor{mypink}{RGB}{255,181,197}
\definecolor{mygray}{gray}{.9}
\definecolor{cyan}{RGB}{193,227,244}
\definecolor{DeepPink}{RGB}{255,20,147}
\newtheorem{theorem}{Theorem}

\newtheorem{proposition}{Proposition}
\newtheorem{corollary}{Corollary}
\newtheorem{definition}{Definition}

\title{Spectral Perturbation Meets Incomplete Multi-view Data}

\author{
Hao Wang$^{1,3}$\and
Linlin Zong$^2$\and
Bing Liu$^3$\and
Yan Yang$^{1}$\footnote{Corresponding author.}\And
Wei Zhou$^1$\\
\affiliations
$^1$School of Information Science and Technology,
Southwest Jiaotong University, Chengdu, China\\
$^2$School of Software, Dalian University of Technology, Dalian, China\\
$^3$Department of Computer Science,
University of Illinois at Chicago, Chicago, USA\\
\emails
hwang@my.swjtu.edu.cn, llzong@dlut.edu.cn, liub@uic.edu, yyang@swjtu.edu.cn
}

\begin{document}

\maketitle

\begin{abstract}
Beyond existing multi-view clustering, this paper studies a more realistic clustering scenario, referred to as \textit{incomplete multi-view clustering}, where a number of data instances are missing in certain views. To tackle this problem, we explore spectral perturbation theory. In this work, we show a strong link between perturbation risk bounds and incomplete multi-view clustering. That is, as the similarity matrix fed into spectral clustering is a quantity bounded in magnitude $\mathcal{O}(1)$, we transfer the missing problem from data to similarity and tailor a matrix completion method for incomplete similarity matrix. Moreover, we show that the minimization of perturbation risk bounds among different views maximizes the final fusion result across all views. This provides a solid fusion criteria for multi-view data. We motivate and propose a Perturbation-oriented \textit{Incomplete multi-view Clustering} (PIC) method. Experimental results demonstrate the effectiveness of the proposed method.
\end{abstract}

\section{Introduction}\label{sec:1}
Many applications face the situation where each data instance in a set $\{\mathbf{x}_1, ..., \mathbf{x}_n\}$ is sampled from multiple views. Here each $\mathbf{x}_i|_{i=1}^n$ is denoted by multiple views, e.g., $m$ views $\{\mathbf{x}_i^1, ..., \mathbf{x}_i^m\}$. Such forms of data are referred to as multi-view data. Multi-view clustering aims to provide a more accurate and stable partition than single view clustering by considering data from multi-views \cite{chao2017survey,Yang2018survey}. To date, most existing multi-view clustering methods, even the most recent methods such as \cite{TaoHLLYZ18,ZongZLY18,Nie2018TL,wang2019gmc,huang2019auto} work under the assumption that every data instance is sampled from all views. We call this assumption the \textit{complete sampling} assumption. However, the complete sampling assumption is too strong as it frequently happens that some data instances are not sampled in certain views because of sensor faults or machine malfunctions. This leads to the result that the collected multi-view data are incomplete in some views. We call such data \textit{incomplete multi-view data}.

The problem of clustering incomplete multi-view data is known as \textit{incomplete multi-view clustering} (or \textit{partial multi-view clustering}) \cite{hu2018doubly,Li2014partial}. Based on the existing works, the main challenge of this problem is 2-fold: (1) how to partition each instance with $m$ views into its group, and (2) how to deal with incomplete views. To address these two challenges, existing incomplete multi-view clustering methods built upon non-negative matrix factorization, kernel learning or spectral clustering to learn a consensus
representation for all views and tackled incomplete views by exploring two main directions. The first direction is to project each incomplete view data into a common subspace or a specific subspace \cite{Li2014partial,zhao2016incomplete,yin2017unified,zhao2018incomplete,cai2018partial,wang2018partial}. However, these methods only work for two-view data. The second direction is to fill the missing instances using matrix completion \cite{xu2015multi,Zhu2018LLZ,wen2018incomplete,liu2018late,shao2015multiple,hu2018doubly,zhou2019consensus}. Most of them still fill the missing features with average feature values. However, such a filling method is naive as the features in both inter-class and intra-class may have large variances. In addition, most existing approaches only evaluate their clustering performance on toy (randomly generated) incomplete multi-view data. 


To address the limitations discussed above, this paper builds a strong link between the spectral perturbation theory and incomplete multi-view clustering. Specifically, we propose a new approach, denoted by \underline{P}erturbation-oriented \underline{I}ncomplete multi-view \underline{C}lustering (PIC). It transfers \textit{feature-value missing} to \textit{similarity-value missing} and reduces the spectral perturbation risk among different views to generate the final clustering results by exploiting the key characteristics of spectral clustering. The proposed approach consists of two main phases.

\textit{Phase 1 is similarity matrix completion}. Given the data matrix of each view, it first generates a similarity matrix (or affinity matrix) for each view. Then it completes the missing similarity entries using average similarity values of other views which have those missing instances.

\textit{Phase 2 is consensus matrix learning}. It first computes the Laplacian matrix of each completed similarity matrix, and then weights each Laplacian matrix using the perturbation theory to learn a consensus Laplacian matrix. Finally, it performs clustering on the consensus Laplacian matrix.

The proposed method PIC can work because spectral clustering partitions data instances according to their similarities, where the similarity value of any two data instances is a quantity bounded in magnitude $\mathcal{O}(1)$. Another crucial point is that the perturbation of spectral clustering is determined by the eigenvectors of the Laplacian matrix, which can be measured by the canonical angle between the subspaces of different eigenvectors. Thus, we can reduce the perturbations among different views by optimizing the canonical angle. We will discuss the details in the subsequent sections.

In summary, this paper makes the following contributions. (1) It proposes a novel incomplete multi-view clustering method by exploiting the spectral perturbation theory. The proposed method transfers feature missing to similarity missing and weights the Laplacian matrix of each view based on perturbation theory to learn a consensus Laplacian matrix for the final clustering. To our knowledge, this is the first such formulation. (2) It provides an upper bound of the spectral perturbation risk among different views and formulates a key task in the proposed model into a standard quadratic programming problem. (3) It experimentally evaluates the proposed method on both toy/synthetic incomplete multi-view data and real-life incomplete multi-view data. The experimental results show that the proposed method makes considerable improvement over the state-of-the-art baselines.

Before going further, we explain some notational conventions used throughout the paper. We will use boldface capital letters (e.g., $\mathbf{X}$), boldface lowercase letters (e.g., $\mathbf{x}$) and lowercase letters (e.g., $x$) to denote matrices, vectors and scalars, respectively. Further, $\mathbf{I}$ denotes the identity matrix, and $\mathbf{1}$ denotes a column vector with all the entries as one. For a matrix $\mathbf{X}\in\mathbb{R}^{n_1\times n_2}$, the $j$-th column vector and the $ij$-th entry are denoted by $\mathbf{x_{j}}$ and $x_{ij}$, respectively. The trace and the Frobenius norm of $\mathbf{X}$ are denoted by $Tr(\mathbf{X})$ and $\|\mathbf{X}\|_F$, respectively. For a column vector $\mathbf{x}\in\mathbb{R}^{n_1\times 1}$, the $j$-th entry is denoted by $x_j$, and $l_p$-norm is denoted by $\|\mathbf{x}\|_p$.

\section{Preliminaries}\label{sec:2}

In this paper, we build upon the work of \cite{NgJW01} (denoted as NgSC), which analyzed the spectral clustering algorithm using the top $k$ eigenvectors of the Laplacian matrix of the similarity matrix to partition data. Given a single view data matrix $\mathbf{X}\in \mathbb{R}^{d\times n}$, where $d$ is the dimension of features and $n$ is the number of data instances, NgSC partitions the $n$ data instances into $c$ clusters as follows:
{\small
	\begin{enumerate}[Step 1.]
		\item Construct the data similarity matrix $\mathbf{A}\in \mathbb{R}^{n\times n}$, where each entry $a_{ij}$ in $\mathbf{A}$ denotes the relationship between $\mathbf{x}_i$ and $\mathbf{x}_j$;
		\vspace{-0.05cm}
		\item Compute the normalized graph Laplacian matrix $\mathbf{L}=\mathbf{D}^{-1/2}\mathbf{A}^T\mathbf{D}^{-1/2}$, where $\mathbf{D}$ is a diagonal matrix whose $i$-th diagonal element is $\sum\nolimits_ja_{ij}$;
		\vspace{-0.05cm}
		\item Let $\lambda_1\!\ge\!...\!\ge\!\lambda_k$ be the $k$ largest eigenvalues of $\mathbf{L}$ and $\mathbf{u}_1,...,\mathbf{u}_k$ denote the corresponding eigenvectors. Normalize all eigenvectors to have unit length and form the matrix $\mathbf{U}\!=\![\mathbf{u}_1,...,\mathbf{u}_k]$ by stacking the eigenvectors in columns;
		\vspace{-0.05cm}
		\item Form the matrix $\mathbf{Y}$ from $\mathbf{U}$ by normalizing each of $\mathbf{U}$'s rows to have unit length;
		\vspace{-0.05cm}
		\item Treat each row of $\mathbf{Y}$ as a data instance, and partition them using K-means to produce the final clustering results.
	\end{enumerate}}
	
	To explain why the eigenvectors of spectral clustering can work, \cite{NgJW01} gave an ``ideal'' case to explain it according to the following proposition.
	\begin{proposition}\label{prop1} \cite{NgJW01}\footnote{Here we denote $\mathbf{A}$ and $\mathbf{Y}$ as $\mathbf{\hat{A}}$ and $\mathbf{\hat{Y}}$ respectively as it is an ``ideal'' case.} Given $n$ data instances with $c$ clusters of sizes $\hat{n}_1, ..., \hat{n}_c$ respectively, let the off-diagonal blocks $\mathbf{\hat{A}}^{(ij)}$ be zero. Also assume that each cluster is connected. Then there exist $k$ ($k\!=\!c$) orthogonal vectors $\mathbf{u}_1, ..., \mathbf{u}_k$ ($\mathbf{u}_i^T\mathbf{u}_j\!=\!1$ if $i\!=\!j$, $0$ otherwise) so that each row of $\mathbf{\hat{Y}}$ satisfies $\mathbf{\hat{y}}_j^{(i)}\!=\!\mathbf{u}_i$ for all $i\!=\!1,...,k$ and $j\!=\!1,...,n_i$.
	\end{proposition}
	
	Proposition \ref{prop1} states that there are $k$ ($k\!=\!c$) mutually orthogonal points on the surface of the unit $k$-sphere around which $\mathbf{\hat{Y}}$'s rows will cluster. These clusters correspond exactly to the true clustering results of the original data.
	
	However, in a general case, the off-diagonal blocks $\mathbf{A}^{(ij)}$ are non-zero. Suppose $\mathbf{E} \!=\!\mathbf{A}\!+\!\mathbf{\hat{A}}$ as perturbations to the ``ideal'' $\mathbf{\hat{A}}$ that makes $\mathbf{A}\!=\!\mathbf{\hat{A}}\!+\!\mathbf{E}$. Earlier results \cite{hunter2010performance} have shown that small perturbations in the similarity matrix can affect the spectral coordinates and clustering ability. The results are based on the following proposition by \cite{hunter2010performance}.
	\begin{proposition}\label{prop2}
		Suppose $\|a_{ij}-\hat{a}_{ij}\|\le\epsilon$, then
		\begin{equation*}
			\|\mathbf{A}-\mathbf{\hat{A}}\|\le n\epsilon.
		\end{equation*}
	\end{proposition}
	
	Proposition \ref{prop2} is a reformulation of the Corollary 10 in \cite{hunter2010performance}. It is easy to prove this proposition as follows
	\begin{equation*}
		\|\mathbf{A}-\mathbf{\hat{A}}\|\!=\!\sqrt{\sum_{ij}(a_{ij}-\hat{a}_{ij})}\le\sqrt{\sum_{ij}\epsilon^2}=\sqrt{n^2\epsilon^2}=n\epsilon.
	\end{equation*}
	
	As can be seen, if $n$ is very large (even $\epsilon$ is small), then $n\epsilon$ cannot be ignored. This problem is more acute in multi-view setting because the constructed similarity matrices of different views may vary greatly. Given all the above, in an incomplete multi-view data clustering setting, we ask the following questions:
	\begin{itemize}
		\item How to handle incomplete multi-view data?
		\vspace{-0.05cm}
		\item How to find a consensus matrix $\mathbf{Y}^*$ for all views?
		\vspace{-0.05cm}
		\item How to make the resulting rows of $\mathbf{Y}^*$ to cluster similarly to the rows of $\mathbf{\hat{Y}}^*$?
	\end{itemize}
	
	The first one is our key question. We will propose our solutions to these questions in the next section.

\section{Proposed Method}
This section presents the proposed PIC method together with its optimization algorithm. 

\subsection{Similarity Matrix Generation}
Given a set of unlabeled data instances $\{\mathbf{x}_1, ..., \mathbf{x}_n\}$ sampled from $m$ views, let $\mathbf{X}^1, ..., \mathbf{X}^m$ be the data matrices of the $m$ views and $\mathbf{X}^v\!=\!\{\mathbf{x}_1^v, ..., \mathbf{x}_{n_v}^v\}\!\in\! \mathbb{R}^{d_v\times n_v}$ be the $v$-th view data matrix, where $d_v$ is the dimension of the features and $n_v$ ($n_v\!\leq\!n$) is the number of data instances. Like most existing work, we make the assumption that at least one view is available for each data instance in the data matrix. Now we generate each view's similarity matrix from each view's data matrix respectively. As each view is independent in this phase, we take the $v$-th view as an example.

The intuition here is that if two data instances are close, they should be also close to each other in the similarity graph. Thus, we propose to learn a similarity matrix as follows
\begin{equation}\label{eq:1}
	\begin{aligned}
		\min_{\mathbf{A}} &\sum_{i,j=1}^n \|\mathbf{x}_i - \mathbf{x}_j\|_2^2\ a_{ij}+\alpha \sum_{i=1}^n \|\mathbf{a}_i\|_2^2 \\
		&s.t.\ a_{ii}=0,\ 0\leq a_{ij} \leq 1,\ \mathbf{1}^T\mathbf{a}_i^v=1.
	\end{aligned}
\end{equation}

The above optimization, i.e., Eq. \eqref{eq:1}, is able to learn a similarity matrix (whose size is $n\times n$ as there are $n$ instances) from a complete data matrix (whose size is also $n\times n$). However, it cannot learn such a similarity matrix from an incomplete data matrix (whose size is not $n\times n$). Our $\mathbf{X}^v$ falls in this case as some instances may be missing in view $v$, resulting in $n_v\leq n$. To handle missing instances, we define a missing operator on each instance $\mathbf{x}_i$ as below
\begin{equation*}
	\mathcal{P}_M(\mathbf{x}_i^v)\stackrel{def}{=}
	\begin{cases}
		\mathbf{x}_i^v, & \text{if $\mathbf{x}_i$ is sampled in the $v$-th view;}\vspace{0.05cm}\\
		NaN, & \text{otherwise;}\\
	\end{cases}
\end{equation*}
where $NaN$ denotes ``not a number'', which can be seen as an invalid number. Based on $\mathcal{P}_M(\mathbf{x}_i^v)|_{i=1}^n$, we formulate our similarity matrix generation task as follows
\begin{equation}\label{eq:2}
	\begin{aligned}
		\min_{\mathbf{A}^v} &\sum_{i,j=1}^n \|\mathcal{P}_M(\mathbf{x}_i^v) - \mathcal{P}_M(\mathbf{x}_j^v)\|_2^2\ a_{ij}^v+\alpha \sum_{i=1}^n \|\mathbf{a}_i^v\|_2^2 \\
		&s.t.\ a_{ii}^v=0,\ 0\leq a_{ij}^v \leq 1,\ \mathbf{1}^T\mathbf{a}_i^v=1.
	\end{aligned}
\end{equation}

In such a way, Eq. \eqref{eq:2} can learn a similarity matrix $\mathbf{A}^v\!\in\! \mathbb{R}^{n\times n}$ with adaptive neighbors for each view. More precisely, it assigns $NaN$ to $a_{ij}^v$ if either $\mathbf{x}_i$ or $\mathbf{x}_j$ is missing in view $v$; otherwise it assigns a similarity value to $a_{ij}^v$ using the following solution. 

We denote $d_{ij}^v\!=\!\|\mathcal{P}_M(\mathbf{x}_i^v)\!-\!\mathcal{P}_M(\mathbf{x}_j^v)\|_2^2$ and further denote $\mathbf{d}_i^v$ as a vector with $j$-th element as $d_{ij}^v$. Here we assign $NaN$ to $d_{ij}^v$ if either $\mathcal{P}_M(\mathbf{x}_i^v)\!=\!NaN$ or $\mathcal{P}_M(\mathbf{x}_j^v)\!=\!NaN$. Then we rewrite Eq. \eqref{eq:2} in a vector form as follows,
\begin{equation}\label{eq:3}
	\mathop{\min }\limits_{\mathbf{a}_i^v} \left\|\mathbf{a}_i^v\!+\!\frac{\mathbf{d}_i^v}{2\alpha}\right\|_2^2,\ s.t.\ a_{ii}^v\!=\!0,\ 0\!\leq\! a_{ij}^v \!\leq\! 1,\ \mathbf{1}^T\mathbf{a}_i^v\!=\!1.
\end{equation}

This problem can be solved with a closed form solution as introduced in \cite{NieWJH16}. So, we generate a similarity matrix $\mathbf{A}^v\in \mathbb{R}^{n\times n}$ for each view $v$ ($v\!=\!1,...,m$). We show that the similarity value of any two data instances (except for missing instances) is a quantity bounded in magnitude $\mathcal{O}(1)$ because we make the constraint $0\!\leq\!a_{ij}^v \!\leq\!1$. This allows for completing those $NaN$s using average similarity values to reduces to perturbations from missing instances.

\subsection{Similarity Matrix Completion}
Given the learned similarity matrices $\mathbf{A}^1,...,\mathbf{A}^m$, we now focus on completing those $NaN$s in each similarity matrix. Specifically, we complete those $NaN$s using the average similarity values of the valid view(s). We also take the $v$-th view as an example. Similar to $\mathcal{P}_M(\mathbf{x}_i^v)$, we define a completion operator on each $\mathbf{a}_i^v$ as
\begin{equation*}
	\mathcal{P}_\Omega(\mathbf{a}_i^v)\stackrel{def}{=}
	\begin{cases}
		\mathbf{a}_i^v, & \text{if every item in $\mathbf{a}_i^v$ is not a $NaN$;}\vspace{0.05cm}\\
		\mathbf{a}_i^{ave}, & \text{otherwise;}\\
	\end{cases}
\end{equation*}
where $\mathbf{a}_i^{ave}\!=\!\textstyle{\sum}_j\mathbf{a}_i^j/N_v$. Here $\mathbf{a}_i^j$ denotes the similarity value vector in the view $j$ (which has valid similarity value vector for the $i$-th entry) and $N_v$ is the number of such views.

According to the following theorem, we discuss why our completion scheme is stable and effective.
\begin{theorem}\label{completion}
	\cite{candes2010matrix}
	Let $\mathbf{Z}\!\in\!\mathbb{R}^{t_1\times t_2}$ be a fixed rank matrix with strong incoherence parameter $\mu$. Suppose there are $\hslash$ observed entries of $\mathbf{Z}$ with locations sampled uniformly at random with noise $\|P_\Omega(\mathbf{Z})-\hat{\mathbf{Z}}\|_F\leq \delta$. Then with high probability, the resulting completion $\hat{\mathbf{Z}}$ obeys
	\begin{equation*}
		\|\mathbf{Z}-\hat{\mathbf{Z}}\|_F\le 4\sqrt{\frac{(2+\hslash) min(t_1, t_2)}{\hslash}}\delta + 2\delta.
	\end{equation*}
\end{theorem}

The details of Theorem \ref{completion} are introduced in \cite{candes2010matrix,hunter2010performance}. The theorem provides an upper bound (which is proportional to the noise level $\delta$) on the recovery error from matrix completion. It states the following: when perfect noiseless recovery occurs, then matrix completion is stable vis-$\grave{a}$-vis perturbations. Our $\mathbf{A}^v$ is a special case of $\mathbf{Z}$ with $t_1\!=\!t_2\!=\!n$. As discussed early, similarity value is a quantity bounded in magnitude $\mathcal{O}(1)$, resulting in a small $\delta$. Thus, our completion scheme using the average similarity values makes completion stable and effective.

Here we have proposed our solution to the key question, i.e., how to handle incomplete multi-vew data. Next, we discuss how to find a consensus matrix $\mathbf{Y}^*$ for all views.

\subsection{Consensus Learning}
Recall the framework of NgSC, see Section \ref{sec:2}. $\mathbf{Y}$ is generated from $\mathbf{U}$ by normalizing each of $\mathbf{U}$'s rows to have unit length. $\mathbf{U}$ is formed by the eigenvectors of Laplacian matrix $\mathbf{L}$. As can be seen, the above processes from $\mathbf{L}$ to $\mathbf{Y}$ are simple yet solid but nothing can be changed. Thus, we transfer learning a consensus $\mathbf{Y}^*$ to learning a consensus $\mathbf{L}^*$. Another crucial reason is that the perturbation of spectral clustering is determined by eigenvector of Laplacian matrix \cite{hunter2010performance}. However, small perturbations in the entries of a Laplacian matrix can lead to large perturbations in the eigenvectors. We will detail this in the next subsection. 

Suppose we have computed the normalized Laplacian matrix $\mathbf{L}^v\!\in\!\mathbb{R}^{n\times k}$ for each completed similarity matrix $\mathbf{A}^v$ using $\mathbf{L}^v\!=\!(\mathbf{D}^v)^{-1/2}(\mathbf{A}^v)^T(\mathbf{D}^v)^{-1/2}$, then we propose to solve our consensus learning task as below
\begin{equation}\label{eq:4}
	\mathbf{L}^*=\sum_{v=1}^m\omega_v\mathbf{L}^v\quad s.t.\ \sum_{v=1}^m\omega_v=1, \omega\geq 0
\end{equation}
where $\omega_v$ is the weight of the $v$-th view. Note that each $\omega_v|_{v=1}^m$ is determined automatically by reducing perturbation risk among different views, which will be clear shortly.

\subsection{Perturbation Risk}
Now we respond to the previous subsection and answer the last question, i.e., how to make the resulting rows of $\mathbf{L}^*$ to cluster similarly to the rows of ``ideal'' $\mathbf{\hat{L}}^*$ as in Proposition \ref{prop2}.

The study in \cite{hunter2010performance} shows that small perturbations in the entries of Laplacian matrix can lead to large perturbations in the eigenvectors. Matrix perturbation theory \cite{Stewart90matrixperturbation} indicates that the perturbations can be captured by the closeness of the subspaces spanned by the eigenvectors. Let $\mathbf{u}_1^v,...,\mathbf{u}_k^v$ and $\mathbf{u}_1^*,...,\mathbf{u}_k^*$ denote the first $k$ eigenvectors of $\mathbf{L}^v$ and $\mathbf{L}^*$, respectively. The subspaces spanned by the eigenvectors $\mathbf{u}_1^v,...,\mathbf{u}_k^v$ and $\mathbf{u}_1^*,...,\mathbf{u}_k^*$ are formed as $[\mathbf{u}_1^v,...,\mathbf{u}_k^v]$ and $[\mathbf{u}_1^*,...,\mathbf{u}_k^*]$. Following \cite{Stewart90matrixperturbation,hunter2010performance}, we define closeness of these supspaces using canonical angles.
\begin{definition}\label{def2}
	Let $\gamma_1\leq...\leq\gamma_k$ be the singular values of $[\mathbf{u}_1^v,...,\mathbf{u}_k^v]^T[\mathbf{u}_1^*,...,\mathbf{u}_k^*]$. Then the values,
	\begin{equation*}
		\theta_i|_{i=1}^k=\arccos \gamma_i
	\end{equation*}
	are called the \textbf{canonical angles} between these subspaces.
\end{definition}
The largest canonical angle indicates the perturbation level. Next we make $\mathbf{L}^*$ close to the ``ideal'' $\mathbf{\hat{L}}^*$ according to the following theorem of canonical angle.
\begin{theorem}\label{theo:2}
	\cite{hunter2010performance}
	Let $\lambda_i^v,\mathbf{u}_i^v,\lambda_i^*,\mathbf{u}_i^*$ be the $i$-th eigenvalue and eigenvector of $\mathbf{L}^v$ and $\mathbf{L}^*$ respectively. Let $\mathbf{\Theta}\!=\!diag(\theta_1,...,\theta_k)$ be the diagonal matrix of canonical angles between the subspaces of $\mathbf{U}^v\!=\![\mathbf{u}_1^v,...,\mathbf{u}_k^v]$ and $\mathbf{U}^*\!=\![\mathbf{u}_1^*,...,\mathbf{u}_k^*]$. If there is a gap $\xi$ such that 
	\begin{equation*}
		|\lambda_k^v-\lambda_{k+1}^*|\geq\xi\quad \text{and}\quad \lambda_k^v\geq\xi
	\end{equation*}
	then
	\begin{equation*}
		\|\sin{\mathbf{\Theta}}\|_F\leq \frac{1}{\xi}\|\mathbf{L}^*\mathbf{U}^v-\mathbf{U}^v\mathbf{\Sigma}^v\|_F
	\end{equation*}
	where $\sin{\mathbf{\Theta}}$ is taken entry-wise and $\mathbf{\Sigma}^v\!=\!diag(\lambda_1^v,...,\lambda_k^v)$.
\end{theorem}
This is a reformulation of the Theorem 3 in \cite{hunter2010performance}. Based on Proposition \ref{prop2} and Theorem \ref{theo:2}, we further give an upper bound of $\sin{\mathbf{\Theta}}$.
\begin{corollary}
	With the notations of Theorem \ref{theo:2}, suppose $\|l_{ij}^*-l_{ij}^v\|\le\epsilon$. Then
	\begin{equation*}
		\|\sin{\mathbf{\Theta}}\|_F\leq \frac{1}{\xi}\sqrt{k}n\epsilon.
	\end{equation*}
\end{corollary}
\begin{proof}
	Recall $
	\|\sin{\mathbf{\Theta}}\|_F\leq \frac{1}{\xi}\|\mathbf{L}^*\mathbf{U}^v-\mathbf{U}^v\mathbf{\Sigma}^v\|_F$.
	
	As the diagonal elements of $\mathbf{\Sigma}^v$ and the column vectors of $\mathbf{U}^v$ are exactly the first $k$ eigenvalues and eigenvectors of $\mathbf{L}^v$ respectively, we have $\mathbf{L}^v\mathbf{U}^v=\mathbf{U}^v\mathbf{\Sigma}^v$. Then
	\begin{equation*}
		\begin{aligned}
			\|\sin{\mathbf{\Theta}}\|_F&\leq \frac{1}{\xi}\|\mathbf{L}^*\mathbf{U}^v-\mathbf{U}^v\mathbf{\Sigma}^v\|_F\!=\!\frac{1}{\xi}\|\mathbf{L}^*\mathbf{U}^v-\mathbf{L}^v\mathbf{U}^v\|_F\\
			&\leq\frac{1}{\xi}\|\mathbf{U}^v\|_F\|\mathbf{L}^*-\mathbf{L}^v\|_F.
		\end{aligned}
	\end{equation*}
	
	According to the orthogonality of $\mathbf{U}^v$, we have
	\begin{equation*}
		\|\mathbf{U}^v\|_F=\sqrt{Tr((\mathbf{U}^v)^T\mathbf{U}^v)}=\sqrt{k}.
	\end{equation*}
	
	Based on Proposition \ref{prop2}, we have $\|\mathbf{L}^*-\mathbf{L}^v\|_F\leq n\epsilon$.
	
	Then, we conclude the proof as
	\begin{equation*}
		\|\sin{\mathbf{\Theta}}\|_F\leq \frac{1}{\xi}\sqrt{k}n\epsilon.
		\vspace{-0.2cm}
	\end{equation*}
	\vspace{-0.1cm}
\end{proof}

Now the key task is to minimize the upper bound of $\sin{\mathbf{\Theta}}$ as it is equivalent to reduce the perturbation risk. Considering the ``ideal'' Laplacian matrix $\mathbf{\hat{L}}^*$, we have the following theorem by \cite{mohar:1991}.

\begin{theorem}\label{theo3}
	\cite{mohar:1991}
	The multiplicity of the eigenvalue $0$ of the Laplacian matrix $\mathbf{\hat{L}}^*$ is exactly equal to the number of clusters $c$.
\end{theorem}

Theorem \ref{theo3} indicates that the ``ideal'' Laplacian matrix $\mathbf{\hat{L}}^*$ has $c$ positive eigenvalues and $n-c$ zero eigenvalues. Let $k\!=\!c$. As we aim to make our $\mathbf{L}^*$ approximate $\mathbf{\hat{L}}^*$, the rank of $\mathbf{L}^*$ is expected to $k$, $\lambda_{k+1}^*\!\approx\!0$ and $|\lambda_k^v-\lambda_{k+1}^*|\!\approx\!\lambda_k^v$. Thus, given the Laplacian matrix of each view, $\xi$ in Theorem \ref{theo:2} can be seen as a constant. This motivates us to rewrite Eq. \eqref{eq:4} into the following objective function to reduce perturbation risk.
\begin{equation}\label{eq:5}
	\begin{aligned}
		\mathop{\min }\limits_{\mathbf{L}^*,\ \mathbf{\omega}}\ &\sum_{v=1}^m\left\|\mathbf{L}^*\mathbf{U}^v-\mathbf{U}^v\mathbf{\Sigma}^v\right\|_F^2\\
		&s.t.\ \textstyle\mathbf{L}^*\!=\!\sum_{v=1}^m\omega_v\mathbf{L}^v,\ \sum_{v=1}^m\omega_v\!=\!1,\ \omega\!\geq\!0.
	\end{aligned}
\end{equation}

Here another motivation is that views having similar clustering ability should be assigned similar weights. According to Definition \ref{def2} and Theorem \ref{theo:2}, we conclude that the largest canonical angle between subspaces spanned by the eigenvectors indicates the similarity of the clustering ability. Thus, the difference in weights between the views should be small if the largest canonical angle between corresponding subspaces is small. This is exactly what manifold learning aims to do \cite{cai2008non}. Let $\psi_{ij}\!\in\![0, \pi]$ be the largest canonical angle between subspaces of the $i$-th view and the $j$-th view and $s_{ij}\!=\!\pi\!-\!\psi_{ij}$. We propose to perform our task using manifold learning as follows
\begin{equation}\label{eq:6}
	\mathop{\min }\limits_{\mathbf{\omega}}\frac{1}{2}\sum_{i,j}^ms_{ij}(\omega_i-\omega_j)^2=\mathop{\min }\limits_{\mathbf{\omega}}\mathbf{\omega}^T\mathbf{H}\mathbf{\omega}
\end{equation}
where $\mathbf{H}=\mathbf{\breve{D}}-\mathbf{S}$ and $\mathbf{\breve{D}}\in\mathbb{R}^{m\times m}$ is a diagonal matrix with each diagonal element as $\breve{d}_{ii}=\sum_{j=1}^ms_{ij}$.

Plugging the right item of Eq. \eqref{eq:6} into Eq. \eqref{eq:5}, formally, our objective function is formulated as
\begin{equation}\label{eq:7}
	\begin{aligned}
		\mathop{\min }\limits_{\mathbf{L}^*,\ \mathbf{\omega}}\ &\sum_{v=1}^m\left\|\mathbf{L}^*\mathbf{U}^v-\mathbf{U}^v\mathbf{\Sigma}^v\right\|_F^2+\beta\mathbf{\omega}^T\mathbf{H}\mathbf{\omega}\\
		&s.t.\ \textstyle\mathbf{L}^*\!=\!\sum_{v=1}^m\omega_v\mathbf{L}^v,\ \sum_{v=1}^m\omega_v\!=\!1,\ \omega\!\geq\!0
	\end{aligned}
\end{equation}
where $\beta$ is a trade-off parameter. We will analyze $\beta$ in the experiment section. Given $\beta$, here we can rewrite Eq. \eqref{eq:7} as
\begin{equation}
	\begin{aligned}\label{eq:8}
		\mathop{\min }\limits_{\mathbf{L}^*,\mathbf{\omega}}\ &\mathbf{\omega}^T(\sum_{v=1}^m\mathbf{Q}^v+\beta\mathbf{I})\mathbf{\omega}-2\mathbf{\omega}^T(\sum_{v=1}^m\mathbf{f}^v)\\
		&s.t.\ \textstyle\mathbf{L}^*\!=\!\sum_{v=1}^m\omega_v\mathbf{L}^v,\ \sum_{v=1}^m\omega_v\!=\!1,\ \omega\!\geq\!0.
	\end{aligned}
\end{equation}

Note, each entry $q_{ij}^v$ in $\mathbf{Q}^v$ and each entry $f_{i}^v$ in $\mathbf{f}^v$ come shortly on the next page. It is easy to see that Eq. \eqref{eq:8} is a standard quadratic programming problem with respect to $\mathbf{\omega}$ and can be solved by a classic technique, e.g., the technique called \textit{quadprog} in MATLAB. We used MATLAB because all baselines used it.

In Eq. \eqref{eq:8},  each entry $q_{ij}^v$ in $\mathbf{Q}^v$ and each entry $f_{i}^v$ in $\mathbf{f}^v$ is respectively defined as
\begin{equation*}
	\small{
		q_{ij}^v\!=\!Tr(\mathbf{L}^i\mathbf{U}^v(\mathbf{U}^v)^T(\mathbf{L}^v)^T), f_{i}^v\!=\!Tr(\mathbf{L}^i\mathbf{U}^v(\mathbf{\Sigma}^v)^T(\mathbf{U}^v)^T).}
	\normalsize
\end{equation*}

Hereto, we presented the proposed PIC approach with four tasks. We now couple overall solutions into a joint framework and optimize the joint framework using Algorithm \ref{alg}. The convergence of our algorithm involves two parts, i.e., generating similarity matrix $\mathbf{A}^v|_{v=1}^m$ using Eq. \eqref{eq:2} and calculating the weights $\mathbf{\omega}$ using Eq. \eqref{eq:8}. Eq. \eqref{eq:2} is clearly a convex function as its second order derivative w.r.t. $\mathbf{a}_i^v$ is a positive value. Eq. \eqref{eq:8} is a standard quadratic programming problem w.r.t. the weights $\mathbf{\omega}$. Thus, the convergence of the proposed algorithm is guaranteed. Compared to single view spectral clustering algorithm, PIC needs to optimize Eq.~\eqref{eq:8}. The computational complexity of optimizing Eq. \eqref{eq:8} is $O(m^3n^2k+m^3)$ in total, where $m\!\ll\! n$ and $k\!\ll\! n$. Thus, PIC does not increase the computational complexity of spectral clustering, i.e., $O(n^3)$. For large-scale data, data sampling as in \cite{CaiC15,li2015large} is a potential way to speed up our method. In addition, the proposed algorithm can be implemented with data on disk as it runs without iterative optimization. Thus, our algorithm can be deployed on a small memory machine.

\begin{algorithm}[t]
	\DontPrintSemicolon
	\SetAlgoLined
	\SetKwInOut{Input}{Input}
	\SetKwInOut{Output}{Output}
	\Input{Data matrices with $m$ views $\mathbf{X}^1,...,\mathbf{X}^m$, the number of clusters $c$, and parameter $\beta$.}
	\Begin{
		Generate similarity matrix $\mathbf{A}^v$ from each data matrix $\mathbf{X}^v$ by solving Eq. \eqref{eq:2};\;
		Complete similarity matrix $\mathbf{A}^v$ using completion operator $\mathcal{P}_\Omega(\mathbf{a}_i^v)|_{i=1}^n$;\;
		Compute normalized Laplacian matrix $\mathbf{L}^v$ of each completed similarity matrix $\mathbf{A}^v$;\;
		Calculate the eigendecomposition $\mathbf{U}^v$ and $\mathbf{\Sigma}^v$ of each normalized Laplacian matrix $\mathbf{L}^v$;\;
		Calculate $\omega$ by solving Eq. \eqref{eq:8};\;
		Calculate consensus Laplacian matrix $\mathbf{L}^*$ using Eq. \eqref{eq:4};\;
		Produce the final clustering results by performing spectral clustering algorithm (e.g., NgSC) on the learned consensus Laplacian matrix $\mathbf{L}^*$;\;
	}
	\Output{The clustering results with $c$ clusters.}
	\caption{The proposed overall algorithm.}
	\label{alg}
\end{algorithm}

\section{Experiments}
\subsection{Datasets and Baselines}
\paragraph{Datasets.} We perform evaluation using four complete multi-view datasets and three natural incomplete multi-view datasets. The datasets are summarized in Table \ref{tab:dataset}, where the first four datasets are complete and the last three datasets are naturally incomplete. For the dataset Mfeat, we collected it from two Handwritten Digits sources, i.e., MNIST and USPS. 

\begin{table}[t]
	\vspace{0.12cm}
	\tiny
	\renewcommand{\arraystretch}{1.1}
	\setlength{\tabcolsep}{5.5pt}
	\centering
	\begin{tabular}{l|ccc|c|c}
		\toprule
		Dataset & $m$ & $c$ & $n$ & $n_v\ (v=1,...,m)$ & $d_v\ (v=1,...,m)$ \\
		\midrule
		100Leaves~\footnotemark[2] & 3 & 100 & 1600 & 1600, 1600, 1600 & 64, 64, 64 \\
		Flowers17~\footnotemark[3] & 7 & 17 & 1360 & 1360, 1360, ..., 1360 & 1360, 1360, ..., 1360  \\
		Mfeat~\footnotemark[4] & 2 & 10 & 10000 & 10000, 10000 & 784, 256  \\
		ORL~\footnotemark[5] & 4 & 40 & 400 & 400, 400, 400, 400 & 256, 256, 256, 256 \\
		\midrule
		3Sources~\footnotemark[6] & 3 & 6 & 416 & 352, 302, 294 & 3560, 3631, 3068 \\
		BBC~\footnotemark[7] & 4 & 5 & 2225 & 1543, 1524, 1574, 1549 & 4659, 4633, 4665, 4684 \\
		BBCSport~\footnotemark[7] & 2 & 5 & 737 & 644, 637 & 3183, 3203 \\
		\bottomrule
	\end{tabular}
	\vspace{-0.12cm}
	\caption{\label{tab:dataset} Summary of the datasets. \{$m$, $c$, $n$, $n$, $n_v$, $d_v$\}: number of \{views, clusters, instances, observed instances, features\} in each view, respectively.}
	\normalsize
\end{table}
\footnotetext[2]{\url{https://archive.ics.uci.edu/ml/datasets/One-hundred+plant+species+leaves+data+set}}
\footnotetext[3]{\url{http://www.robots.ox.ac.uk/~vgg/data/flowers/17/index.html}}
\footnotetext[4]{\url{https://cs.nyu.edu/~roweis/data.html}}
\footnotetext[5]{\url{www.cad.zju.edu.cn/home/dengcai/Data/FaceData.html}}
\footnotetext[6]{\url{http://mlg.ucd.ie/datasets/3sources.html}}
\footnotetext[7]{\url{http://mlg.ucd.ie/datasets/segment.html}}

\paragraph{Baselines.} We consider the following algorithms as the baselines: \textbf{BSV} \footnotemark[8] (Best Single View) \cite{NgJW01}, \textbf{PVC} \cite{Li2014partial}, \textbf{IMG} \cite{zhao2016incomplete}, \textbf{MIC} \cite{shao2015multiple} and \textbf{DAIMC} \cite{hu2018doubly}. Note that BSV only works for complete single view data. Following \cite{shao2015multiple,zhao2016incomplete}, we first fill the missing instance in each incomplete view using the average feature values of that incomplete view. PVC and IMG only work for two-view data. Following \cite{hu2018doubly}, we evaluate PVC and IMG on all  two-view combinations and report the best results. Since DAIMC works for multi-view data, we use it as it is.

\begin{table*}[!tb]
	\renewcommand{\arraystretch}{1.2}
	\setlength{\tabcolsep}{3.5pt}
	\centering
	\resizebox{\textwidth}{!}{
		\begin{tabular}{|l|ccccc>{\columncolor{mypink}}c|ccccc>{\columncolor{mypink}}c|}
			\hline
			\multirow{2}{*}{Method} & \multicolumn{6}{c}{\textbf{Clustering performance in terms of ACC}} & \multicolumn{6}{|c|}{\textbf{Clustering performance in terms of NMI}} \\
			\cline{2-13}
			& BSV & PVC & IMG & MIC & DAIMC & PIC & BSV & PVC & IMG & MIC & DAIMC & PIC  \\
			\hline
			3Sources & 22.50$\pm$0.63 & 26.45$\pm$0.39 & 25.59$\pm$0.13 & 43.73$\pm$6.28 & 58.68$\pm$8.12 & \textbf{88.08$\pm$1.22} & ~~5.30$\pm$0.27 & ~~1.77$\pm$0.25 & ~~2.00$\pm$0.12 & 38.94$\pm$5.58 & 48.80$\pm$7.27 & \textbf{73.50$\pm$1.45} \\
			BBC & 40.79$\pm$1.02 & 37.80$\pm$0.97 & 29.92$\pm$0.01 & 57.07$\pm$9.74 & 51.34$\pm$7.44 & \textbf{87.03$\pm$0.05} & 25.99$\pm$2.33 & 14.72$\pm$0.11 & ~~6.24$\pm$0.01 & 39.19$\pm$6.54 & 37.74$\pm$7.23 & \textbf{70.12$\pm$0.02} \\
			BBCSport & 40.99$\pm$0.81 & 44.33$\pm$1.46 & 37.86$\pm$0.07 & 58.66$\pm$9.39 & 75.16$\pm$8.82 & \textbf{76.02$\pm$5.28} & 26.60$\pm$0.48 & 13.77$\pm$1.67 & ~~7.55$\pm$0.03 & 46.26$\pm$6.74 & 57.80$\pm$9.53 & \textbf{75.12$\pm$2.05} \\
			\hline
		\end{tabular}}
		\vspace{-0.2cm}
		\caption{Clustering performance on three natural incomplete multi-view datasets.}
		\label{tab:toyresult}
	\end{table*}
	
	\begin{figure*}[!htb]
		\vspace{-0.15cm}
		\centering
		\hspace{0.465cm}\includegraphics[width=0.972\textwidth]{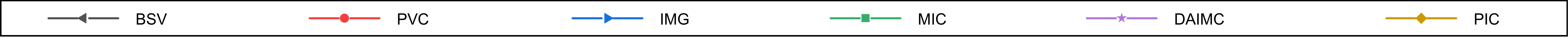}%
		\par%
		\includegraphics[width=0.225\textwidth]{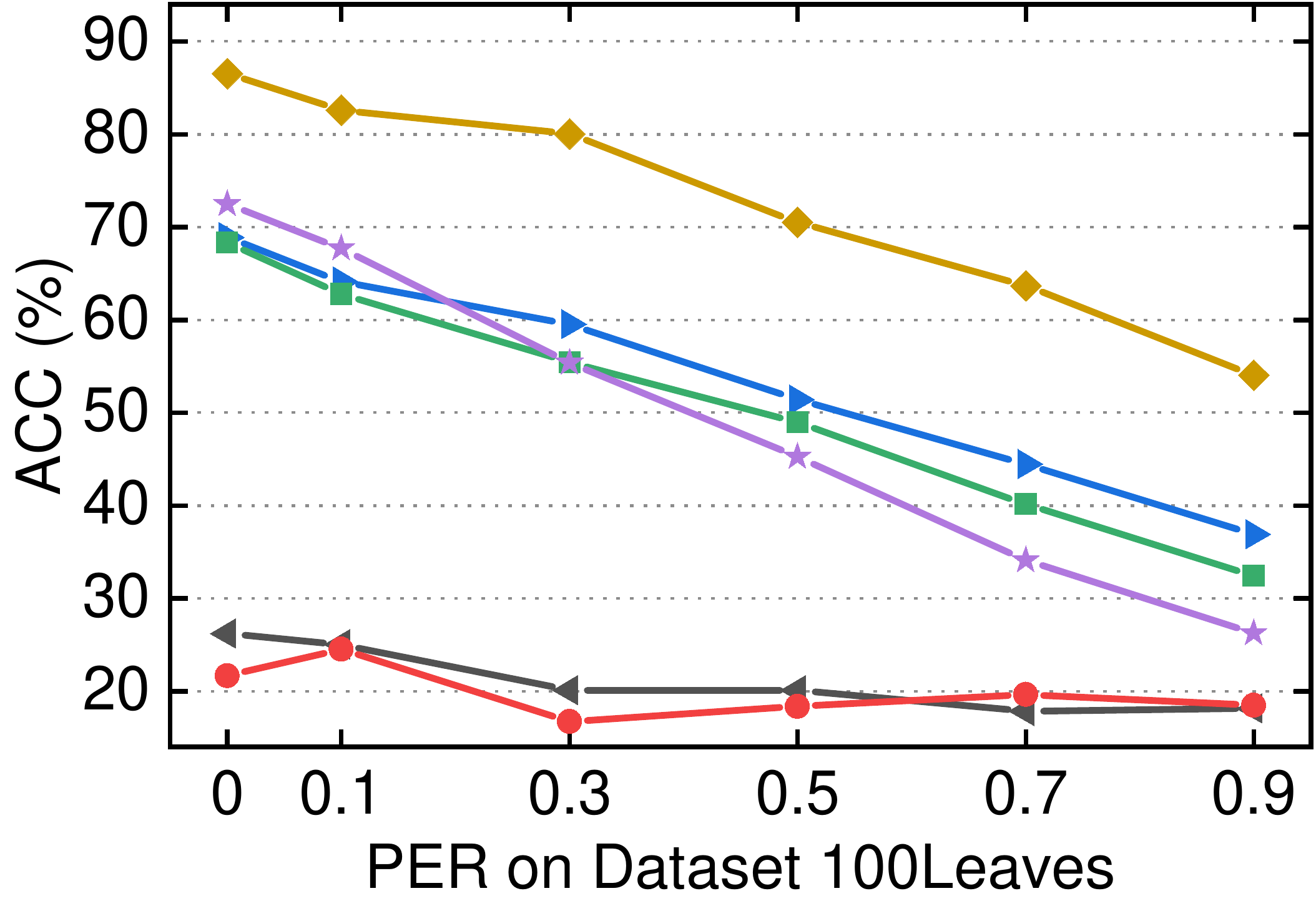}%
		\hfill%
		\includegraphics[width=0.225\textwidth]{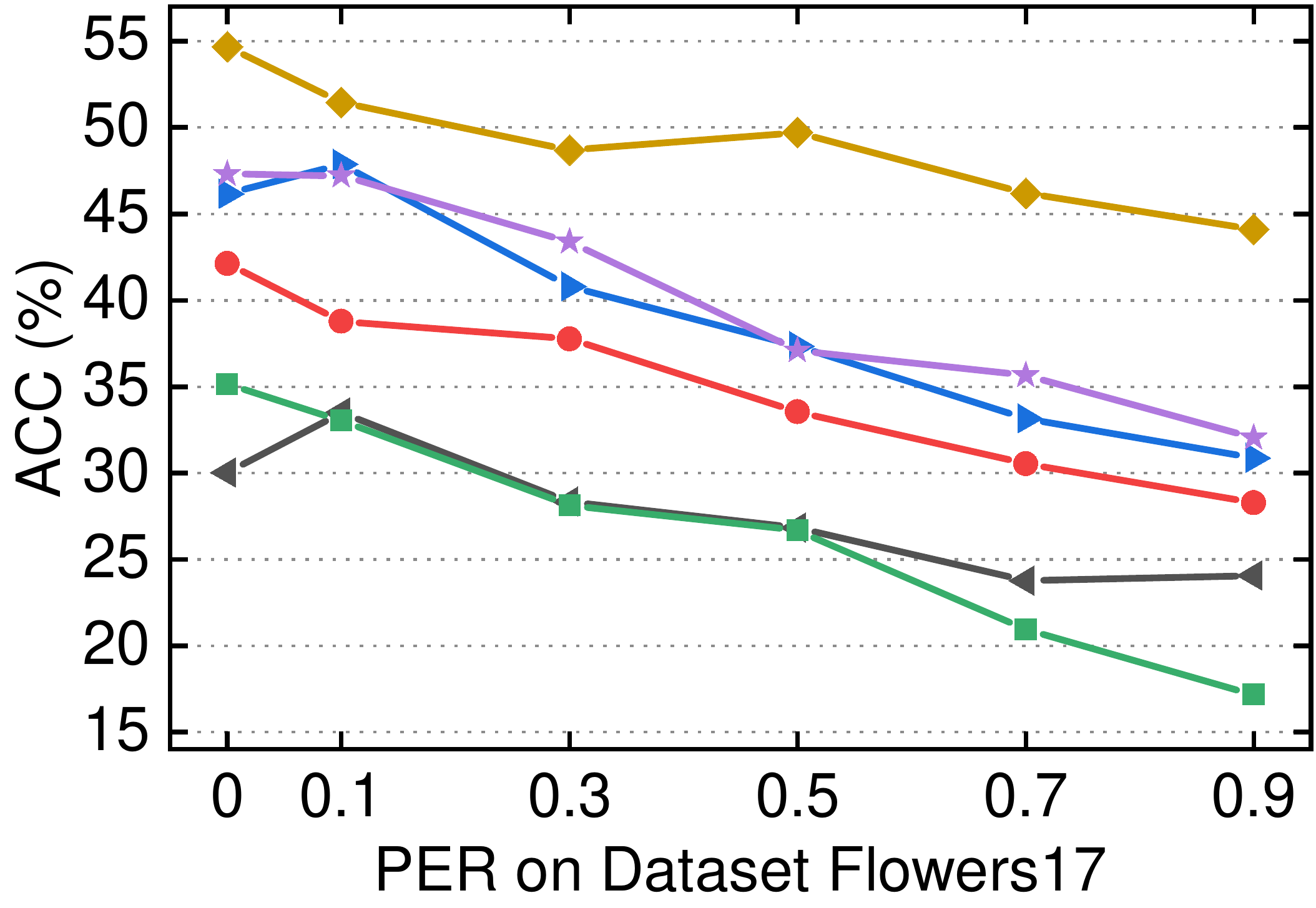}%
		\hfill%
		\includegraphics[width=0.225\textwidth]{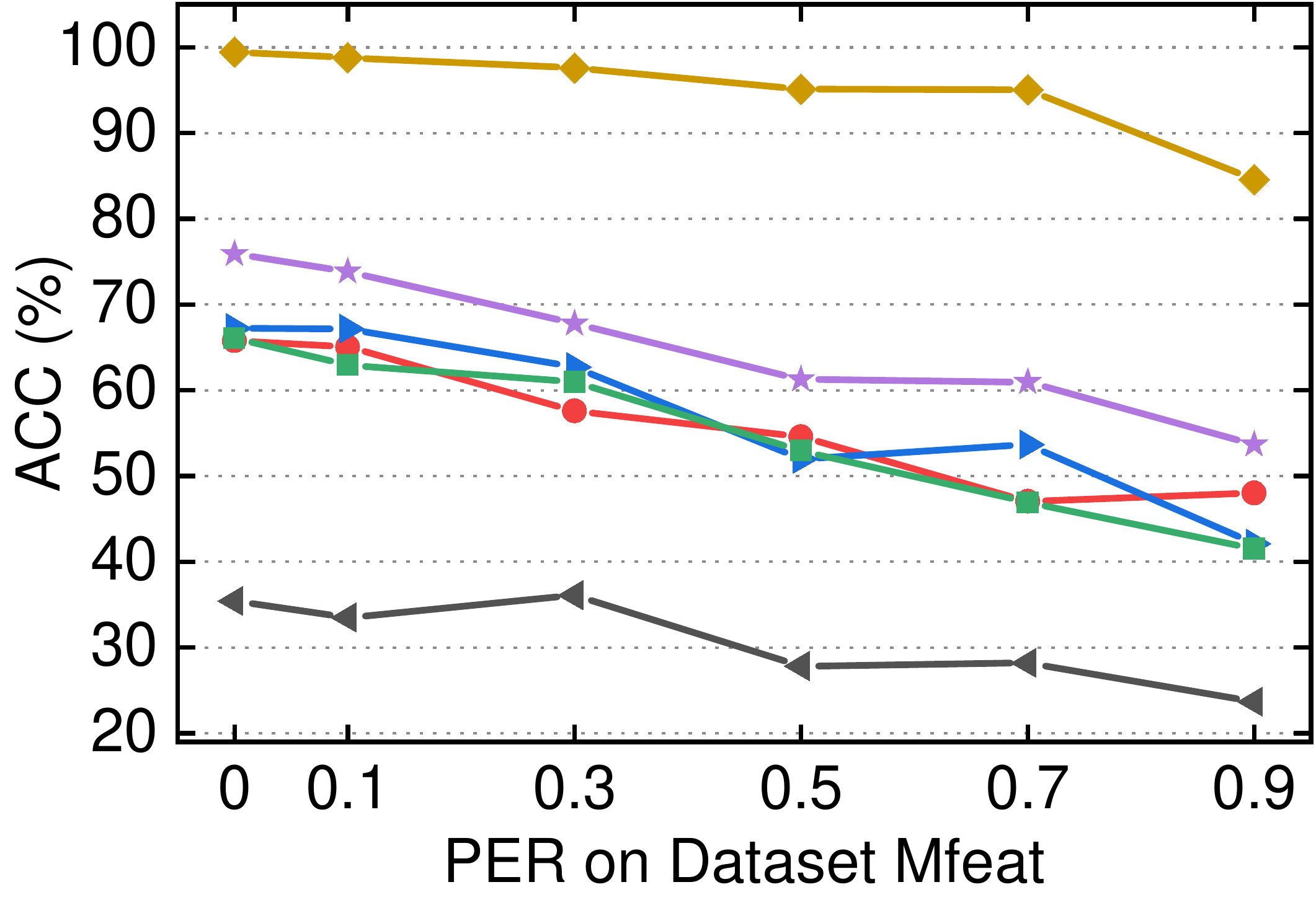}%
		\hfill%
		\includegraphics[width=0.225\textwidth]{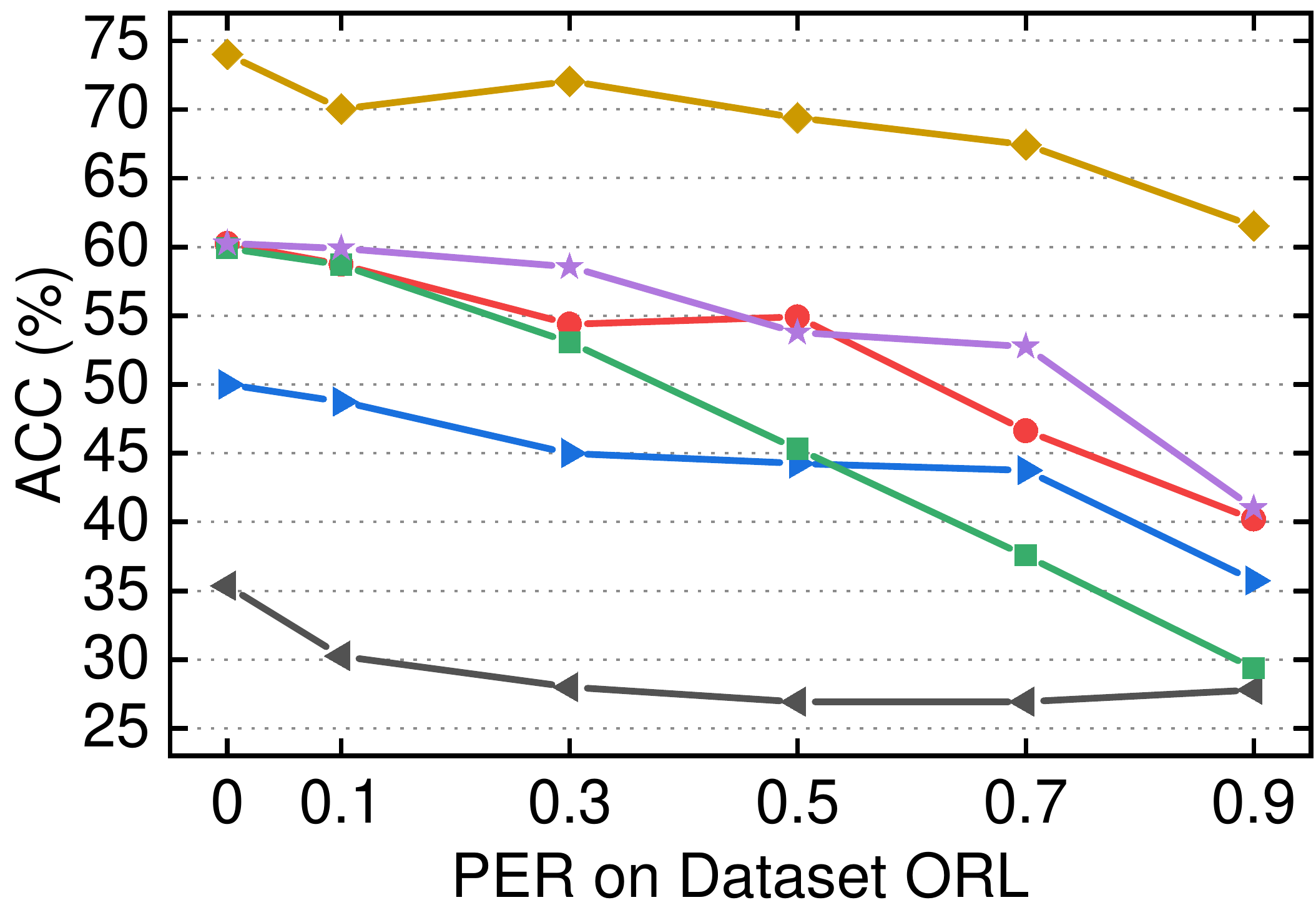}%
		\par%
		\includegraphics[width=0.225\textwidth]{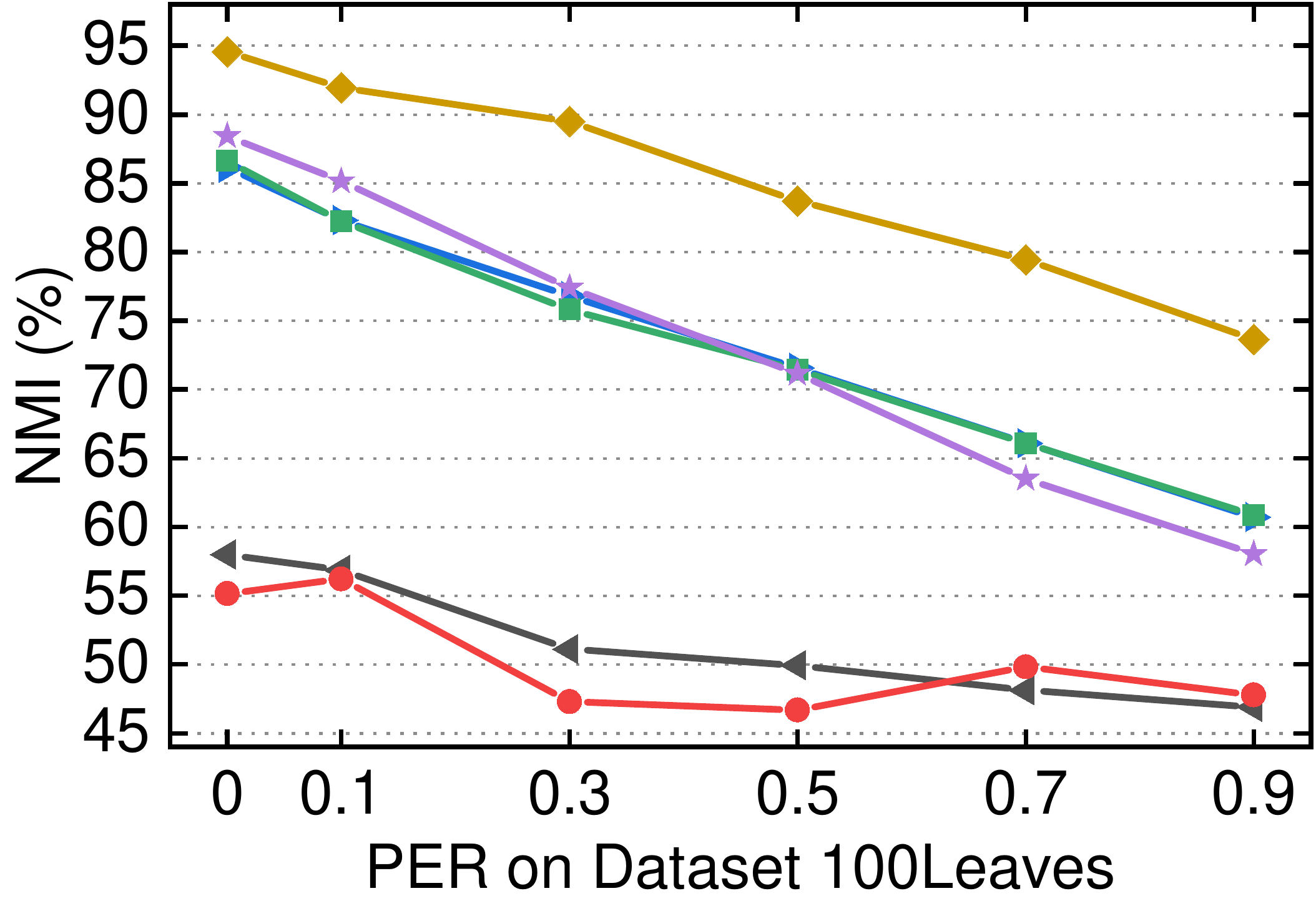}%
		\hfill%
		\includegraphics[width=0.225\textwidth]{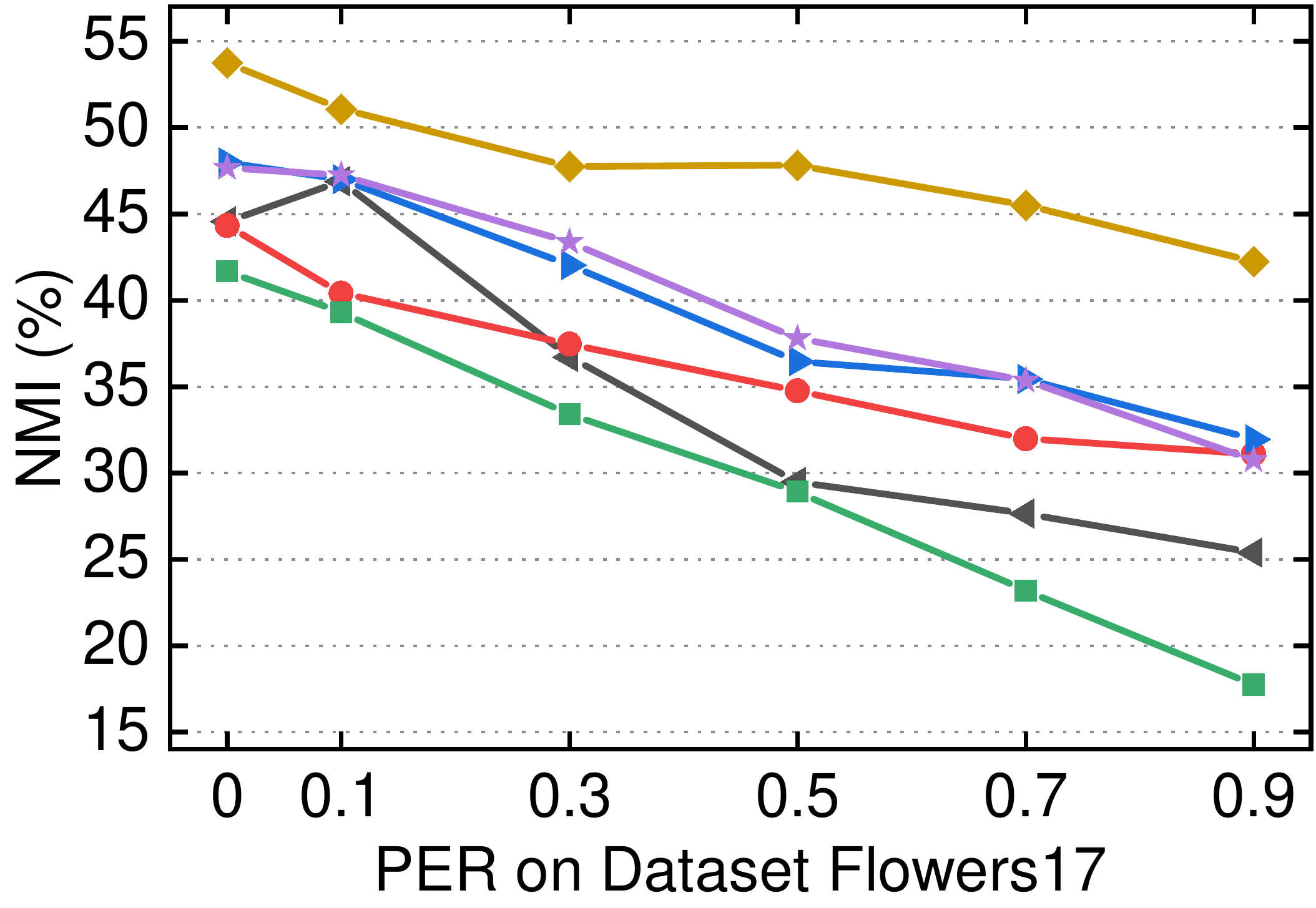}%
		\hfill%
		\includegraphics[width=0.225\textwidth]{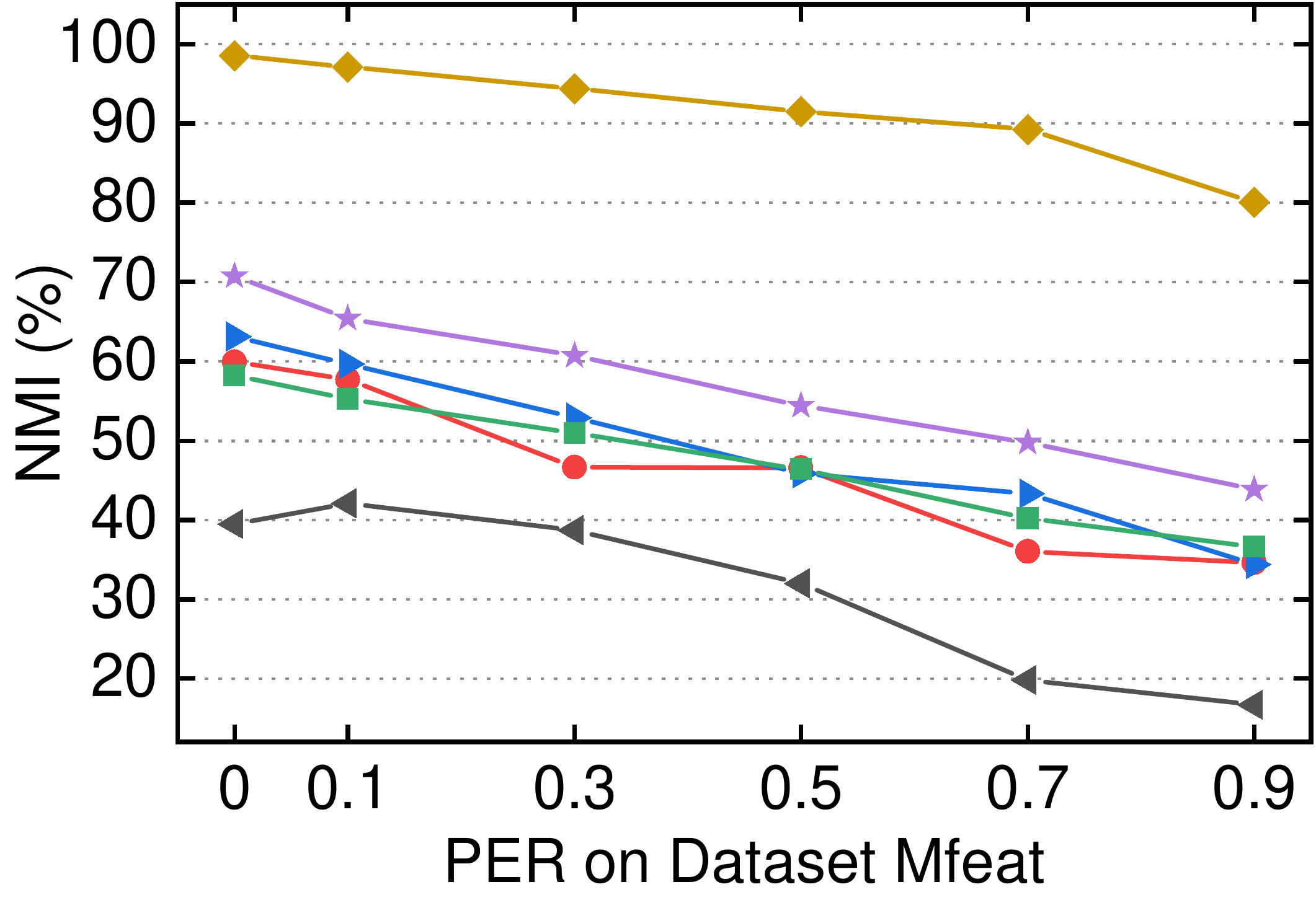}%
		\hfill%
		\includegraphics[width=0.225\textwidth]{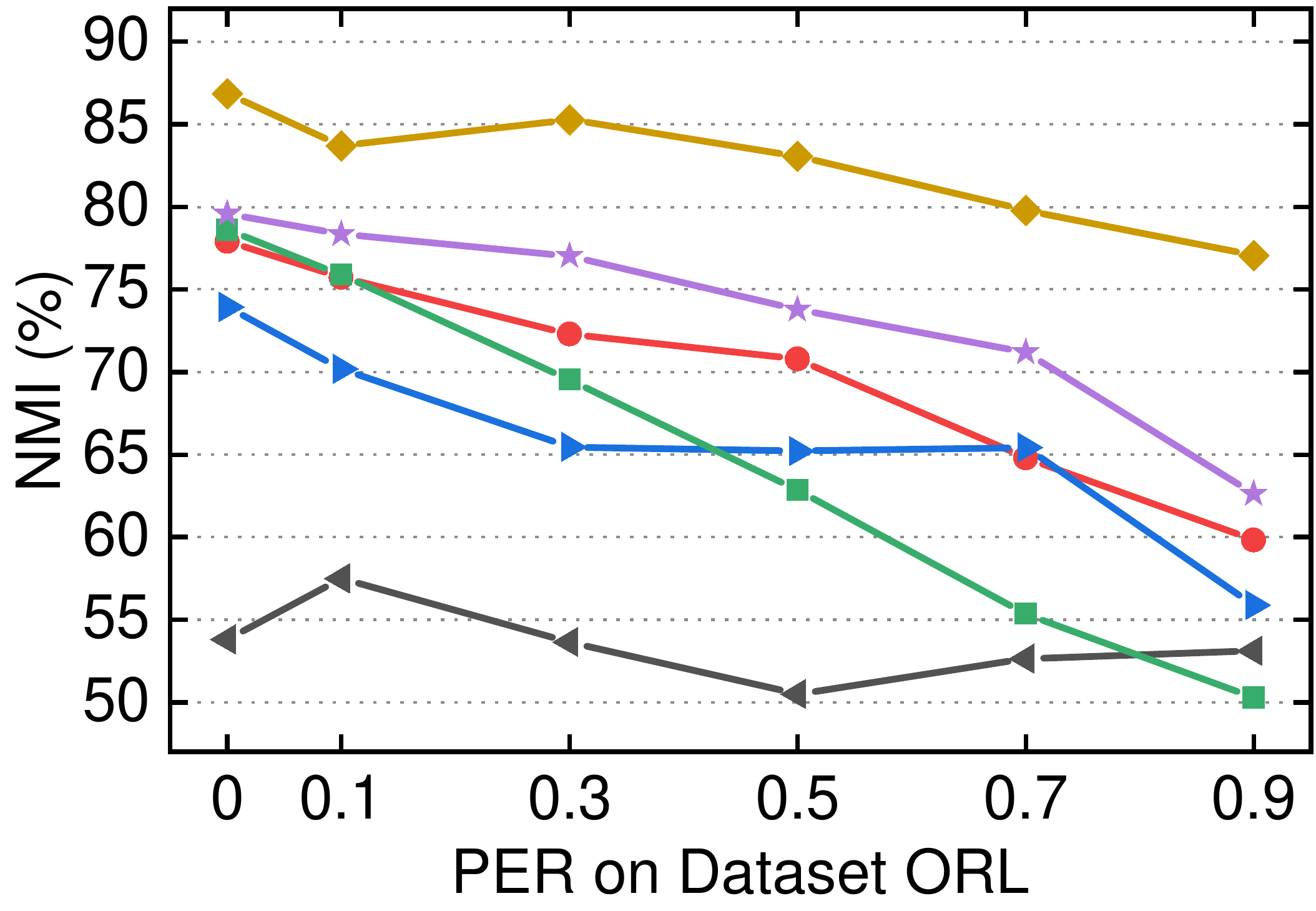}%
		\vspace{-0.2cm}
		\caption{Clustering performance results with different PER settings on four toy incomplete multi-view datasets.}
		\label{Fig1}
		\vspace{-0.2cm}
	\end{figure*}
	\subsection{Experimental Settings and Results}
	\subsubsection{Experimental Settings}
	To generate incomplete multi-view datasets from complete multi-view datasets, all baselines randomly select partial examples/instances under different Partial Example Ratio (PER). Then they evenly distribute these partial examples to each view and delete them from each view. Here we use a general setting. Same as the baselines, we first randomly select partial examples under different PER. Then we additionally generate a random binary vector $\mathbf{b}\!=\!(b_1,...,b_m)$ for each partial example (e.g., $\mathbf{x}_j$). If $b_i\!=\!0$, we delete example $\mathbf{x}_j$ from view $i$. 
	
	For the baselines, we obtained the original systems from their authors and used their default parameter settings. For PIC, we set the parameter $\beta$ using $\beta\!=\!\tilde{\beta}\!\times\!\|\sum_v\mathbf{Q}^v\|_F/\|\mathbf{I}\|_F$ to balance Eq. \ref{eq:5} and Eq. \ref{eq:6}. Then we empirically set $\tilde{\beta}\!=\!0.1$ in evaluation. The parameter study will come shortly.
	
	Following the baselines, two metrics,  accuracy (ACC) and normalized mutual information (NMI) are used to measure the clustering performance. In order to randomize the experiments, we run each algorithm 20 times and report the average values of the performance measures.
	
	\subsubsection{Experimental Results}
	We first perform evaluation using four toy incomplete multi-view datasets (generated from complete multi-view datasets). In this experiment, PER varies from 0.1 to 0.9 with an interval of 0.2, same as baselines PVC and IMG. We also set PER = 0, i.e., every data instance is sampled in all views. Figure \ref{Fig1} shows the performance results in terms of ACC and NMI. From Figure \ref{Fig1}, we make the 
	following observations:
	\begin{itemize}
		\vspace{-0.05cm}
		\item Our PIC significantly outperforms the baselines in all the PER settings. As the PER increases, the clustering performance of all the methods drops.
		\vspace{-0.1cm}
		\item All baselines are inferior to our model. One reason is that they complete the missing instances with the average feature values, which results in a large deviation, especially when the PER is large. In this work, we transfer feature missing to similarity missing, then complete the missing similarity entries (marked with $NaN$) using the average similarity values of all the valid views. This shows that our completion scheme is powerful in dealing with missing instances.
		\vspace{-0.1cm}
		\item The recent baseline DAIMC performs better than other baselines, but worse than our method. This shows that our method presents a new margin to beat.
		\vspace{-0.05cm}
	\end{itemize}
	
	We then perform evaluation using three natural incomplete multi-view data. The results, i.e., the average values and the standard deviations (denoted as ave$\pm$std), are shown in Table \ref{tab:toyresult}. From the table, we make the following observations:
	\begin{itemize}
		\vspace{-0.05cm}
		\item Our PIC again outperforms the baselines markedly. PIC achieves the best ACC and NMI on each dataset. The results clearly show that our PIC is a promising incomplete multi-view clustering method.
		\vspace{-0.1cm}
		\item Multi-view clustering methods (MIC, DAIMC and PIC) are superior to two-view clustering methods (PVC and IMG), but two-view clustering methods (PVC and IMG) are not always superior to the single-view clustering method (BSV). All of them are inferior to our PIC. This indicates that multi-view data boost clustering results with multi-view clustering techniques.
		\vspace{-0.05cm}
	\end{itemize}
	
	
	\subsection{Parameter Study}
	In the above experiments, parameter $\tilde{\beta}$ is set to 0.1 for PIC. Here we explore the effect of parameter $\tilde{\beta}$. Due to the space limit, we only show the results on three natural incomplete multi-view data. Figure \ref{Fig3} shows how the average performance of PIC varies with different $\tilde{\beta}$ values. From Figure \ref{Fig3}, we can see that PIC achieves consistently good performance when $\tilde{\beta}$ is around $0.1$ (i.e., 1e-1) on three datasets. As introduced early, we use a balance scheme, i.e., $\beta\!=\!\tilde{\beta}\!\times\!\|\sum_v\mathbf{Q}^v\|_F/\|\mathbf{I}\|_F$. This is the reason why we can use the same parameter $\tilde{\beta}$ (i.e., $0.1$) for all datasets.
	\begin{figure}[!htb]
		\vspace{-0.15cm}
		\centering
		\hspace{0.37cm}\includegraphics[width=0.436\textwidth]{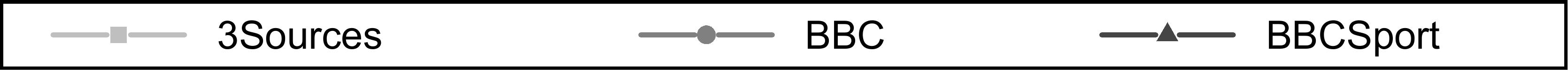}
		\par
		\includegraphics[width=0.225\textwidth]{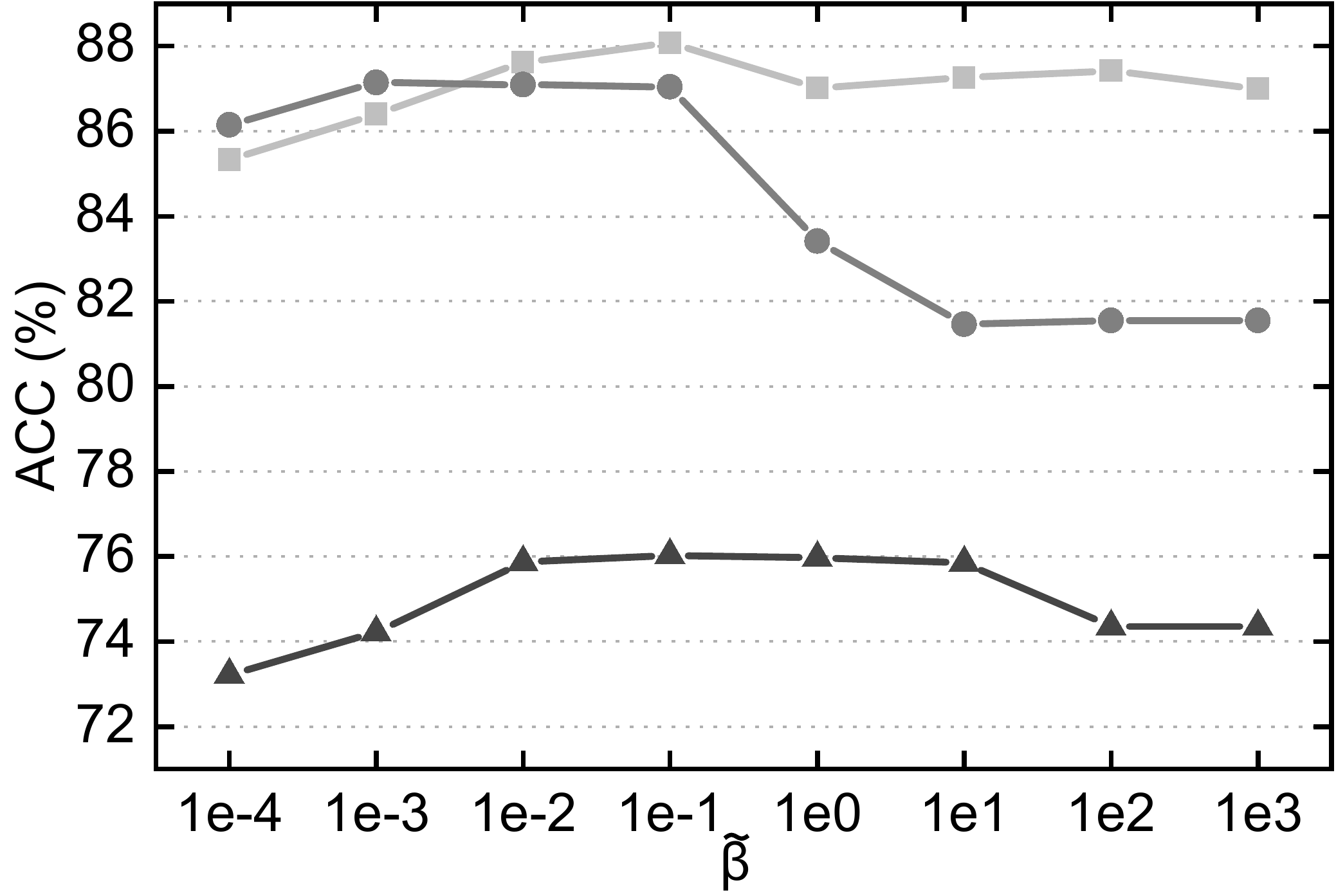}
		\hfil
		\includegraphics[width=0.225\textwidth]{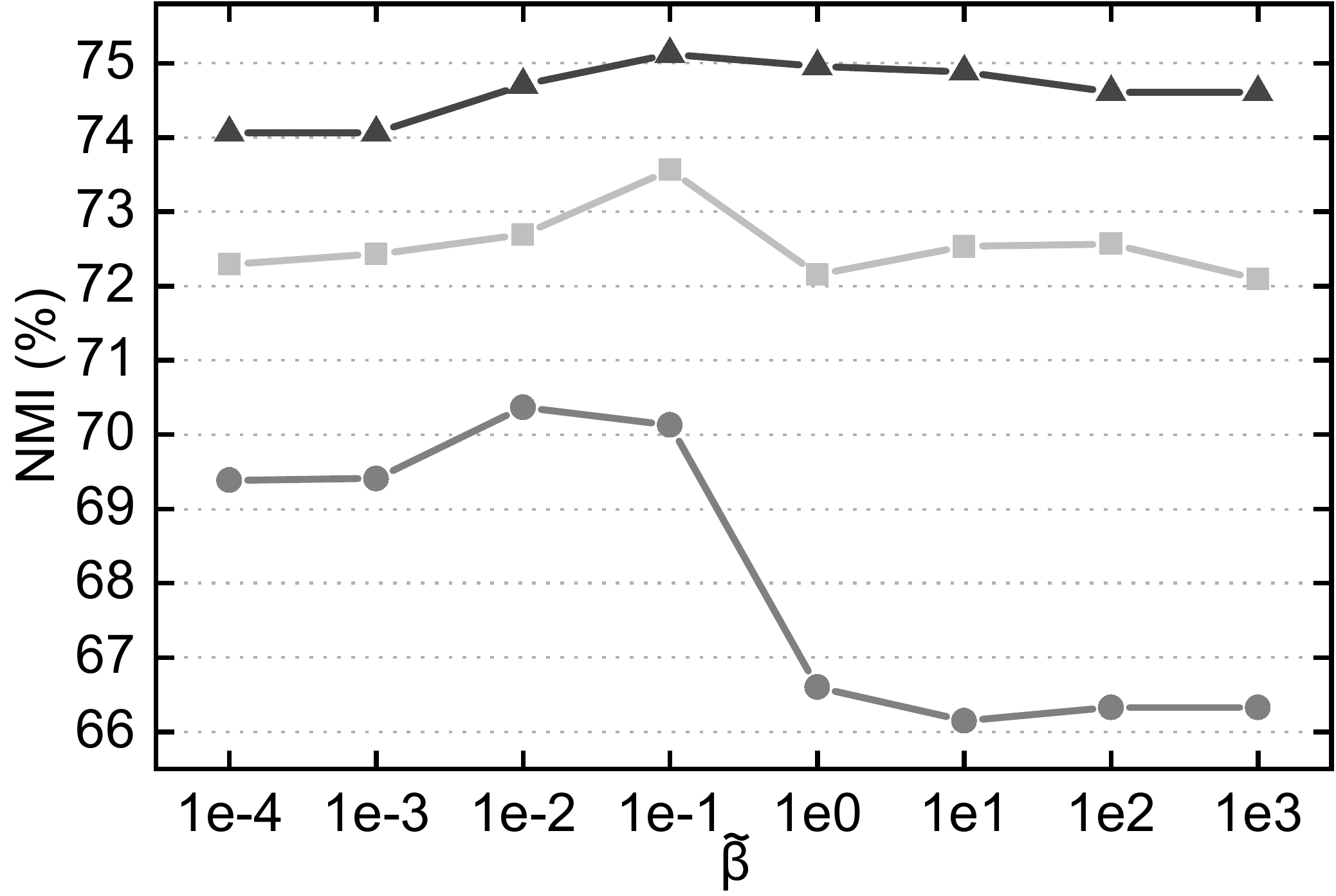}
		\vspace{-0.25cm}
		\caption{\small{Parameter $\tilde{\beta}$ studies on three natural incomplete datasets.}}
		\label{Fig3}
		\vspace{-0.2cm}
	\end{figure}

\section{Conclusions}
This paper built a bridge between spectral perturbation and incomplete multi-view clustering. We explored spectral perturbation theory and proposed a novel Perturbation-oriented Incomplete multi-view Clustering (PIC) method. The key idea is to transfer the missing problem from data matrix to similarity matrix and reduce the spectral perturbation risk among different views while balancing all views to learn a consensus representation for the final clustering results. Both theoretical results and experimental results showed the effectiveness of the proposed method.

\section*{Acknowledgments}
This work was partially supported by the National Natural Science Foundation of China (No. 61572407), and the Project of National Science and Technology Support Program (No. 2015BAH19F02). Hao Wang would like to thank the China Scholarship Council (No. 201707000064).

\bibliographystyle{named}

\end{document}